\documentclass[letterpaper, 10 pt, conference]{ieeeconf}
\IEEEoverridecommandlockouts

\usepackage{cite}
\usepackage{amsmath,amssymb,amsfonts}
\usepackage{algorithmic}
\usepackage{graphicx}
\usepackage{textcomp}
\usepackage[dvipsnames]{xcolor}

\usepackage{booktabs}
\makeatletter
\let\NAT@parse\undefined
\makeatother
\usepackage[bookmarks=true, colorlinks=false, hidelinks, backref=false]{hyperref}
\usepackage[capitalise,noabbrev]{cleveref}
\usepackage{url}
\usepackage[font=footnotesize,labelfont=bf]{caption}
\usepackage{macros}

\def\BibTeX{{\rm B\kern-.05em{\sc i\kern-.025em b}\kern-.08em
    T\kern-.1667em\lower.7ex\hbox{E}\kern-.125emX}}

\usepackage{newtxtext}
\usepackage{newtxmath}
\usepackage{bm}

\title{\LARGE \bf
On the Stability and Realizability of Recurrent \\Polynomial
Surrogate Ternary Logic Gate Networks
}

\author{Sai Sandeep Damera, Ryan Matheu, Aniruddh G. Puranic, John S. Baras and Calin Belta%
\thanks{The authors are with the University of Maryland, College Park, USA. Emails: \{\texttt{sdamera}, \texttt{rmatheu}, \texttt{puranic}, \texttt{baras}, \texttt{calin}\}\texttt{@umd.edu}.}%
}

\begin{document}

\maketitle
\thispagestyle{empty}
\pagestyle{empty}

\begin{abstract}
Recurrent Neural Networks (RNNs) can learn to predict Signal Temporal Logic (STL)
verdicts online from partial trajectories, but deploying them as runtime
monitors in safety-critical systems demands more than predictive accuracy.
When sensor inputs degrade or become unavailable, standard architectures such
as RNNs offer no structural guarantee that their outputs degrade
gracefully; a dropped input can silently flip a verdict from safe to unsafe.
We introduce the Recurrent Differentiable Ternary Logic Gate Network
(R-DTLGN), a recurrent architecture that operates over Kleene's three-valued
logic $\bm{\{-1, 0, +1\}}$, where $\bm{0}$ explicitly represents \emph{unknown}. The
R-DTLGN trains through continuous polynomial surrogates and hardens to a
discrete ternary logic circuit at inference.
We analyze the hardened circuit through two gate vocabularies derived from two
orderings on the ternary domain: numerically monotone gates ensure stable
recurrent dynamics, while information-monotone gates, when present, guarantee
principled abstention (unknown inputs never produce wrong outputs) and
monotonicity in input certainty (more information can only improve the
verdict). We show that the recurrent connections required by bounded STL
operators use exclusively AND and OR, which belong to both vocabularies,
linking the monitoring task's structure to the architecture's guarantees.
A realizability bound derived from the STL formula's temporal operators
directly sizes the network's hidden state, replacing hyperparameter search
with a formula-driven specification. We evaluate these results on STL
specifications over D4RL PointMaze navigation data, testing prediction
accuracy, degradation behavior under predicate dropout, and the
accuracy-versus-safety tradeoff between two label construction pipelines.
The R-DTLGN is, to our knowledge, the first recurrent architecture that
couples learned temporal prediction with formal degradation guarantees rooted
in three-valued logic.
\end{abstract}

\section{Introduction}\label{rdtlgn:sec:introduction}
Monitoring a safety specification online requires evaluating temporal
properties as data arrives. Signal Temporal Logic
(STL)~\cite{maler2004monitoring} provides a formal language for such
properties over real-valued signals, and its quantitative robustness
semantics assign a continuous score measuring the degree of satisfaction
or violation. For bounded temporal operators (e.g.,
$\square_{[0,5]}\,\mu_{\mathrm{safe}}$), the ground-truth robustness at
time $t$ depends on signal values up to $t{+}5$, which a causal monitor
cannot observe. Tools such as STLCG++~\cite{kapoor2025stlcg++} compute the exact robustness offline; when restricted to the available observations, they define a non-learned baseline that can evaluate only the portion of the
specification whose temporal horizon has already elapsed.

\begin{figure}
    \centering
    \includegraphics[width=1\linewidth]{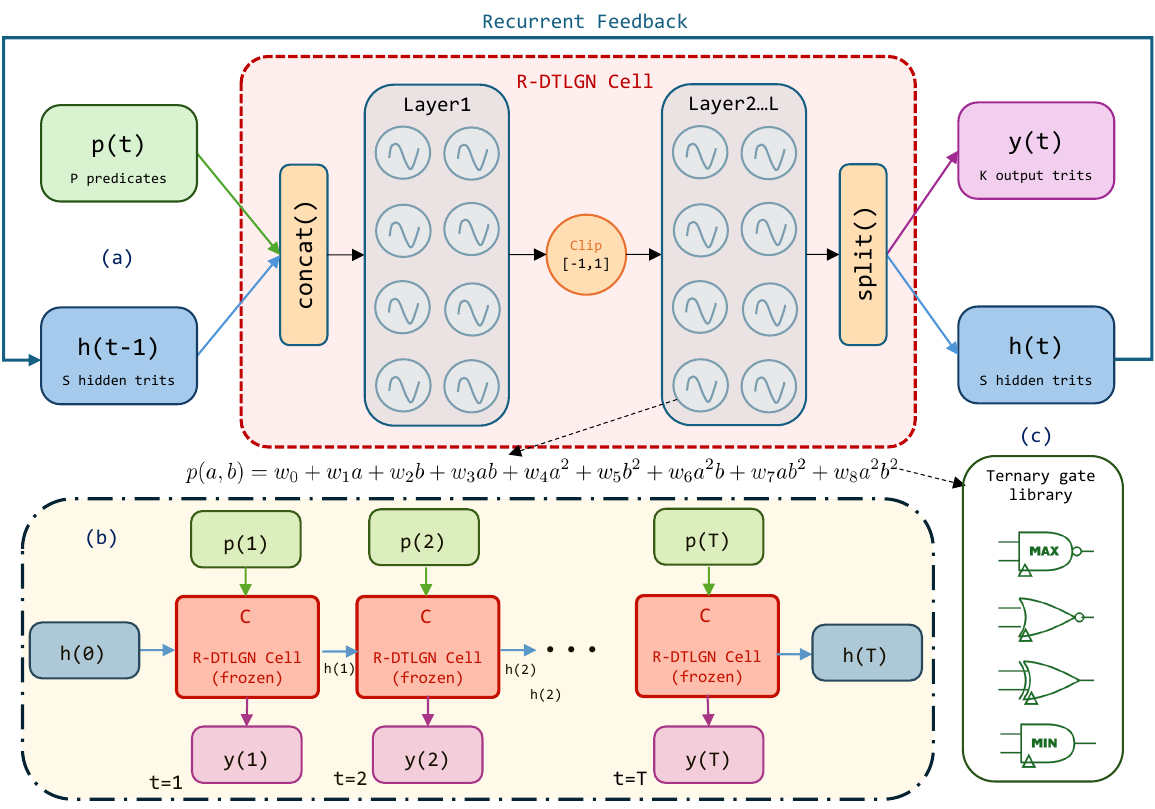}
    \caption{Schematic of the Recurrent Differentiable Ternary Logic Gate Network. (a) The R-DTLGN cell. (b) Unrolled cell. (c) Hardening into Logic Circuit.}
    \label{fig:r-dtlgn}
    \vspace{-0.25cm}
\end{figure}

The gap between this causal baseline and the oracle (full-trajectory)
evaluation is fundamental: it exists for every bounded formula with a
non-trivial future time horizon. Any \emph{learned} monitor attempts to
bridge this gap by exploiting spatio-temporal patterns in training data
to predict verdicts that a causal evaluator cannot resolve from the
available observations alone. Standard recurrent architectures such as
LSTM and GRU can serve
this role, but they provide no structural guarantees about what happens
when input information degrades. Their continuous-valued hidden states
carry no inherent distinction between ``determined'' and ``uncertain,''
so degraded inputs (sensor dropout, indeterminate predicates) can
produce arbitrarily wrong outputs.

Kleene's strong three-valued logic~\cite{fitting1994kleene} is the
natural formalism for reasoning about temporal properties under partial
information. It extends classical Boolean logic with a third value,
$0$, representing ``unknown.'' The Kleene connectives (AND, OR, NOT)
propagate this unknown value conservatively: a conjunction with one
false input is false regardless of the other, but a conjunction of true
and unknown yields unknown. This conservative propagation is precisely
what a monitor operating under input uncertainty should exhibit.

A learned monitor that operates natively in Kleene's three-valued domain
acquires two complementary capabilities. First, like any learned monitor,
it exploits temporal patterns in training data to \emph{predict} verdicts
that a causal evaluator cannot resolve; this prediction capability is
what makes it useful over the causal baseline. Second, unlike standard
recurrent architectures, it provides structural \emph{degradation
guarantees}: when input information is lost (predicates become
indeterminate, sensors fail), the monitor's verdicts degrade toward
``unknown'' rather than flipping to wrong answers. The prediction
capability bridges the causal gap; the degradation guarantee ensures
robustness to input-level uncertainty. These two capabilities are
complementary and non-overlapping.

We propose the Recurrent Differentiable Ternary Logic Gate Network
(R-DTLGN) as a learned monitor that operates natively in Kleene's
domain. Differentiable Logic Gate Networks
(DLGNs)~\cite{petersen2022deep,petersen2024convolutional} replace
conventional arithmetic neurons with discrete two-input logic gates,
producing pure logic circuits at inference. Polynomial Surrogate
Training (PST)~\cite{damera2026polynomialsurrogatetrainingdifferentiable} 
extends DLGNs to Kleene's ternary domain $\Tset = \{-1, 0, +1\}$ by
representing each neuron as a degree-$(2,2)$ polynomial with 9 learnable
coefficients; at inference, the network \emph{hardens} to an exact
ternary logic circuit (Section~\ref{rdtlgn:subsec:pst}). The R-DTLGN
augments this feedforward architecture with a persistent hidden state
for temporal reasoning: at each timestep the cell ingests the current
predicate values and its previous state, processes them through $L$
layers of ternary PST gates, and produces an updated state and output
verdict. The hidden state is initialized to $(0,\ldots,0)$, the
all-unknown state, so the monitor begins with no prior commitment.
Because every neuron in the hardened circuit is a ternary gate on
$\Tset$, the R-DTLGN's inference-time computation is a composition of
Kleene connectives, inheriting the conservative uncertainty propagation
that standard RNNs lack.

\paragraph*{Contributions}
We analyze the R-DTLGN as a causal STL monitor and provide four contributions:

\noindent\textbf{1)~The R-DTLGN architecture.} We introduce the Recurrent
    Differentiable Ternary Logic Gate Network, a recurrent cell whose
    neurons are degree-$(2,2)$ polynomial surrogates during training and
    exact ternary logic gates at inference. Operating natively in Kleene's
    three-valued domain, the R-DTLGN combines the predictive capacity of
    a learned monitor with the conservative uncertainty propagation of
    ternary logic. We introduce a novel hardening routine for converting the soft learned network into logic gate circuits in the recurrent setting based on trajectory distillation.

    \noindent\textbf{2)~Degradation guarantees from Kleene structure.} We prove
    that R-DTLGN cells composed of information-monotone gates satisfy two per-timestep
    properties: \emph{principled abstention} (all-unknown inputs produce
    all-unknown outputs) and \emph{input certainty monotonicity}
    (replacing a determined input with unknown can only make outputs less
    determined, never flip them). These properties guarantee graceful
    degradation under input loss, a structural advantage over standard recurrent
    monitors.

    \noindent\textbf{3)~Realizability bounds from STL formula structure.} Given
    a bounded STL formula $\varphi$, we derive a lower bound $B(\varphi)$
    on the cell's hidden state dimension, decomposed additively by
    temporal operator type. This connects STL formula complexity to neural
    architecture sizing, replacing hyperparameter search with a
    formula-driven specification.

    \noindent\textbf{4)~Empirical validation.} On a suite of 6 bounded STL
    specifications over PointMaze navigation data, we evaluate both
    \emph{prediction performance} (R-DTLGN vs.\ causal STLCG++ vs.\
    vanilla RNN) and \emph{degradation behavior} (predicate dropout
    experiments confirming principled abstention and input certainty
    monotonicity in the learned cell).
\section{Related Work}
\label{rdtlgn:sec:related}

\subsection{STL Monitoring and the Causal Gap}
\label{rdtlgn:subsec:rel-stl}
Offline STL monitoring tools such as
Breach~\cite{donze2010breach}, S-TaLiRo~\cite{annpureddy2011s}, 
STLCG++~\cite{kapoor2025stlcg++} compute quantitative robustness given 
the complete signal trajectory. Online monitoring algorithms process the 
signal incrementally but still require future samples to resolve bounded 
temporal operators: the robustness of $\square_{[0,b]}\,\mu$ at time $t$ 
depends on values up to $t{+}b$. Deshmukh et al.~\cite{deshmukh2017robust}
formalized robust online monitoring of partial traces, where the
monitor may lack sufficient data to determine the Boolean verdict.
Several recent approaches train neural networks to predict STL
robustness from partial trajectories~\cite{ma2020stlnet} using LSTM
or GRU cells as the recurrent backbone. These learned monitors bridge the causal gap
by exploiting spatio-temporal patterns in data, but they lack structural
guarantees about their behavior under input degradation.

\subsection{Three-Valued Logic in Formal Verification}
\label{rdtlgn:subsec:rel-3val}

Kleene's strong three-valued logic~\cite{fitting1994kleene} extends
classical Boolean logic with a third truth value representing
``unknown''. Its connectives are the tightest sound extension of
Boolean operations to partial information: a conjunction of true and
unknown yields unknown, but a conjunction with false yields false
regardless of the other operand. This conservative propagation
principle has made three-valued logic a standard tool in formal
verification, where it arises whenever an abstraction or partial
observation leaves the truth value of a proposition undetermined.
The algebraic structure is that of a De Morgan
lattice~\cite{cignoli1975injective}, with two distinct orderings:
a numerical (truth) ordering and an information (knowledge) ordering.
The information ordering, where $0$ (unknown) is the least element,
is the one relevant to degradation analysis.
This paper brings the same algebraic structure to \emph{learned}
monitors: the R-DTLGN operates natively in Kleene's domain,
inheriting conservative uncertainty propagation as a structural
property rather than an approximation artifact.

\subsection{Differentiable Logic Gate Networks}
\label{rdtlgn:subsec:rel-dlgn}

Differentiable Logic Gate Networks
(DLGNs)~\cite{petersen2022deep} replace conventional arithmetic neurons
with discrete two-input logic gates, training through continuous
relaxations of a categorical gate distribution. Each neuron maintains a
softmax distribution over the 16 Boolean gates, selecting the mode at
inference. Convolutional extensions~\cite{petersen2024convolutional} and
subsequent work on connection optimization demonstrate competitive
accuracy at extreme parameter efficiency. The softmax-over-gates paradigm is feasible because the
Boolean gate space is small ($2^{2^2} = 16$), but it becomes
intractable for ternary logic ($3^{3^2} = 19{,}683$ gates). Polynomial
Surrogate Training (PST)~\cite{damera2026polynomialsurrogatetrainingdifferentiable} 
overcomes this via a direct polynomial parameterization
(Section~\ref{rdtlgn:subsec:pst}), matching binary DLGN accuracy
while training $2$--$7\times$ faster. B\"uhrer et
al.~\cite{buhrer2025recurrent} introduced recurrent DLGNs for sequential
tasks over Boolean values but do not address the hardness fragility that arises in recurrent settings. All prior DLGN work is either feedforward or
operates in the binary domain. Our proposed R-DTLGN combines
recurrence with the ternary PST framework, and its dynamics under the
information ordering have not been previously analyzed.

\paragraph*{Gap addressed by this work} No existing learned STL monitor has a structural 
connection to the three-valued logic that underlies STL semantics under partial
information. The R-DTLGN is, to our knowledge, the first learned
architecture whose inference-time computation operates natively in
Kleene's three-valued domain, providing provable degradation guarantees
that complement its learned prediction capability.

\section{Preliminaries}\label{rdtlgn:sec:prelim}

\subsection{Signal Temporal Logic (STL)}\label{rdtlgn:subsec:stl}

A discrete-time signal is a function $x:\N\to\R^n$ mapping discrete-time indices to a real-valued system state, where $\N$ is the set of all nonnegative integers. STL~\cite{maler2004monitoring} is a modal logic for specifying temporal behaviors of real-valued signals over bounded time intervals. The primitives in STL are \emph{predicate functions}, responsible for mapping a signal to the Boolean domain. A predicate function has the general form $\mu:=f(x(t))\geq0$, where $f:\R^n\to\mathbb{R}$ is a continuous function. STL formulas are defined recursively over the following grammar:
\begin{equation*}\label{rdtlgn:eq:stl-syntax}
    \varphi \;:=\; \mu \,\mid\, \neg\varphi \,\mid\,
    \varphi_1 \wedge \varphi_2 \,\mid\,
    \square_{[a,b]}\,\varphi \,\mid\,
    \lozenge_{[a,b]}\,\varphi \,\mid\,
    \varphi_1\, \mathbf{U}_{[a,b]}\, \varphi_2
\end{equation*}
where $\square_{[a,b]}$ is the unary \emph{always} operator, $\lozenge_{[a,b]}$ is the unary \emph{eventually} operator, $\mathbf{U}_{[a,b]}$ is the binary \emph{until} operator, and $[a,\,b]$ is a bounded time interval with $a, b \in \N$, $a \leq b$. For example, the formula $(x_1(t)\leq 0)\mathbf{U}_{[0,\,5]}\square_{[0,\,2]}(x_2(t)\geq2)$ states that $x_1$ must remain less than or equal to zero until, within 0 to 5 time steps, $x_2$ becomes and remains greater than or equal to two for 3 time steps.

The \emph{quantitative robustness} $\rho(\varphi, x, t) \in \R$
measures the degree of satisfaction at time $t$:
    \begin{equation}
        \label{rdtlgn:eq:stl-robustness}
        \begin{aligned}
            \rho(\mu,\,x,\,t) &= f(x(t)), \\
            \rho(\neg\varphi,\,x,\,t) &= -\rho(\varphi,\,x,\,t), \\
            \rho(\varphi_1 \wedge \varphi_2,\,x,\,t)
                &= \min\bigl(\rho(\varphi_1,\,x,\,t),\; \rho(\varphi_2,\,x,\,t)\bigr), \\
            \rho(\square_{[a,b]}\,\varphi,\,x,\,t)
                &= \min_{\tau \in [t+a,\, t+b]} \rho(\varphi,\,x,\,\tau), \\
            \rho(\lozenge_{[a,b]}\,\varphi,\,x,\,t)
                &= \max_{\tau \in [t+a,\, t+b]} \rho(\varphi,\,x,\,\tau),
        \end{aligned}
    \end{equation}
A positive robustness ($\rho > 0$) implies satisfaction; a negative
robustness ($\rho < 0$) implies violation. The sign of $\rho$ is
sound: $\sign(\rho(\varphi, x, t)) = +1$ if and only if the Boolean
semantics evaluate to true.

For bounded operators, evaluating $\rho$ at time $t$ requires signal
values up to $t{+}b$. A causal monitor operating at time $t$ has access
only to the available observations $\{x(0), \ldots, x(t)\}$ and therefore cannot
compute the oracle robustness for formulas with non-trivial future
horizons.

\subsection{The Ternary Domain and Its Two Orderings}
\label{rdtlgn:subsec:orderings}

The ternary domain $\Tset = \{-1, 0, +1\}$ admits two natural partial orders with different roles. The numerical ordering governs training-time stability; the information ordering governs
inference-time degradation guarantees.

\begin{definition}[Numerical ordering]\label{rdtlgn:def:numerical-ordering}
The \emph{numerical ordering} $\leq$ on $\Tset$ is the restriction of
the standard order on $\R$: $-1 \leq 0 \leq +1$. This is a total
order with bottom $-1$, top $+1$, meet $\min$, and join $\max$.
\end{definition}

\begin{definition}[Information ordering]\label{rdtlgn:def:info-ordering}
The \emph{information ordering} (Kleene ordering) $\sqsubseteq$ on
$\Tset$ is defined by:
$0 \sqsubseteq -1$, $0 \sqsubseteq +1$, with $-1$ and $+1$
incomparable. The value $0$ represents \emph{unknown} and is the
unique bottom element $\bot = 0$. The values $\pm 1$ are maximal but
incomparable: they represent different determined values, not different
levels of certainty. This is a \emph{flat pointed partially ordered set}: an antichain
with a bottom element, but \emph{not} a lattice (the join
$(-1) \vee (+1)$ does not exist).
\end{definition}

Both orderings extend componentwise to $\Tset^n$. In particular,
$\Tset^S$ under $\sqsubseteq$ has bottom
$\bot = (0, \ldots, 0)$ (the ``all-unknown'' state).

\begin{definition}[Monotonicity under each ordering]
\label{rdtlgn:def:monotonicity}
A function $g : \Tset^2 \to \Tset$ is \emph{numerically monotone} if
$a \leq a', b \leq b'$ implies $g(a,b) \leq g(a',b')$, and
\emph{information-monotone} if $a \sqsubseteq a', b \sqsubseteq b'$
implies $g(a,b) \sqsubseteq g(a',b')$.
\end{definition}

The numerical ordering governs the continuous training dynamics: the
cell operates on $[-1,1]^S$ with real-valued activations, and
gradient-based analysis uses the Euclidean structure inherited from
$\R^S$. The information ordering governs the discrete inference
dynamics: after hardening, the cell operates on $\Tset^S$ and the
relevant structure is the pointed poset $(\Tset^S, \sqsubseteq)$ with
bottom $\bot = (0,\ldots,0)$ that
distinguishes ``unknown'' from ``determined''.

\subsection{Gate Vocabularies}\label{rdtlgn:subsec:two-vocabs}

The monotonicity definition (~\cref{rdtlgn:def:monotonicity})
applied under each ordering yields two gate vocabularies with distinct
guarantees. Of the $3^9 = 19{,}683$ two-input ternary gates, exactly
175 are \emph{numerically monotone} (NM) and exactly 197 are
\emph{information-monotone} (IM). Excluding the 3 constant gates leaves
172 NM non-constant and 194 IM non-constant gates.

In a recurrent circuit composed entirely of NM gates, the state-update
map is order-preserving on $[-1,+1]^S$, so Tarski's
theorem~\cite{tarski1955lattice} guarantees fixed-point existence and bounded orbits. NM gates are therefore the natural vocabulary for the recurrent path.
\begin{center}
    \begin{tabular}{c|c}
        $\neg$ & \\
        \hline
        $F$ & $T$ \\
        $U$ & $U$ \\
        $T$ & $F$
    \end{tabular}
    \quad
    \begin{tabular}{c|ccc}
        $\wedge$ & $F$ & $U$ & $T$ \\
        \hline
        $F$ & $F$ & $F$ & $F$ \\
        $U$ & $F$ & $U$ & $U$ \\
        $T$ & $F$ & $U$ & $T$
    \end{tabular}
    \quad
    \begin{tabular}{c|ccc}
        $\vee$ & $F$ & $U$ & $T$ \\
        \hline
        $F$ & $F$ & $U$ & $T$ \\
        $U$ & $U$ & $U$ & $T$ \\
        $T$ & $T$ & $T$ & $T$
    \end{tabular}
\end{center}

The 197 IM gates include all Kleene strong connectives: NOT, AND, OR (defined above),
NAND, NOR; XOR and XNOR under Kleene extension; IMPLIES and its
variants; both projections $\pi_1$ and $\pi_2$; and all three constant
gates. IM gates provide the
degradation guarantees developed in
Section~\ref{rdtlgn:subsec:degradation}. However, the majority of IM gates have
zero-absorbing rows ($g(0,b)=0$ for all $b$). A recurrent circuit of
pure IM gates cascades toward the all-unknown fixed point, producing
degenerate memory.

The two vocabularies overlap at exactly 20 gates (17 non-constant),
including AND, OR, both projections, and their complements. Gates in
this intersection, denoted NM$\,\cap\,$IM, inherit both recurrent
stability (from NM) and degradation guarantees (from IM).
Section~\ref{rdtlgn:sec:analysis} shows that bounded STL temporal
operators require only NM$\,\cap\,$IM gates in their recurrent
connections, making this intersection the ideal target vocabulary.

\subsection{Polynomial Surrogate Training (PST)}\label{rdtlgn:subsec:pst}

We now define the PST neuron, network architecture, training
objective, and hardening procedure used throughout the paper.

\paragraph{PST neuron}
Each two-input ternary gate $g : \Tset^2 \to \Tset$ is represented by
the polynomial $p_{\bm{w}} : \R^2 \to \R$ defined as:
\begin{equation}\label{rdtlgn:eq:pst-polynomial}
    p_{\bm{w}}(a, b)
    = \bm{w}^\top \bm{m}(a,b),
\end{equation}
where
$\bm{m}(a,b) = [1,\, a,\, b,\, ab,\, a^2,\, b^2,\, a^2 b,\,
ab^2,\, a^2 b^2]^\top$ is the monomial basis and
$\bm{w} \in \R^9$ are the learnable coefficients. The $9 \times 9$
Vandermonde matrix $V$ evaluating these monomials at the grid points
$\Tset^2$ is invertible, so the mapping between coefficients $\bm{w}$
and truth table values $\bm{t} = V\bm{w}$ is a bijection. The
polynomial is $C^\infty$-smooth and linear in $\bm{w}$, requiring no 
softmax temperatures or Gumbel relaxations.

\paragraph{Network architecture}
A feedforward PST-DTLGN is specified by its depth $L$, layer widths
$\{n_l\}_{l=0}^{L}$, a connectivity map $\mathbf{C}$ assigning two
parents to each neuron, and the polynomial coefficients
$\bm{w}_j^{(l)} \in \R^9$ for each neuron $j$ at layer $l$. The
forward pass applies a $\clip$ nonlinearity after each polynomial
evaluation:
\begin{equation}\label{rdtlgn:eq:pst-forward}
    h_j^{(l)} = \clip\!\Big(
        p_{\bm{w}_j^{(l)}}\!\big(h_{s_j}^{(\cdot)},\; h_{t_j}^{(\cdot)}\big)
    \Big),
\end{equation}
where $\clip(x) = \max(-1, \min(1, x))$ and $(s_j, t_j)$ are the
parent indices from $\mathbf{C}$. Clipping preserves the range
$[-1,1]$, prevents polynomial blowup across layers, and preserves the
grid points $\{-1,0,+1\}$ exactly.

\paragraph{Training}
The training objective combines a task loss with a commitment
regularizer that encourages polynomial evaluations at grid points to
approach valid ternary values:
\begin{equation}\label{rdtlgn:eq:pst-loss}
    \loss(\mathbf{W}) = \loss_{\mathrm{task}}(\mathbf{W})
    + \lambda(s) \cdot \mathcal{R}_A(\mathbf{W}),
\end{equation}
where $\mathcal{R}_A$ penalizes the distance from each neuron's soft
truth table to the nearest valid ternary truth table, and
$\lambda(s)$ is annealed from $\approx 0$ (free exploration) to
$\lambda_{\max}$ (strong commitment) during training.

\paragraph{Hardening}
At inference, each neuron is converted to a discrete ternary gate: its
polynomial is evaluated on the $3 \times 3$ grid $\Tset^2$, each entry
is rounded to the nearest element of $\Tset$, and the resulting truth
table is looked up in the gate library. Formally, the hardened
coefficients are $\bm{w}_j^{\mathrm{hard}} = V^{-1}
\operatorname{round}_{\Tset}(V \bm{w}_j)$. Since rounding always
produces a valid element of $\Tset^9$, every trained PST neuron
discretizes to one of the $19{,}683$ possible two-input ternary gates
without requiring a curated vocabulary. The hardened network is a pure
ternary logic circuit. This works in the feedforward setting but makes the hardening fragile in the recurrent setting, where per-neuron errors amplify in the recurrent loop.

\section{The Recurrent PST-DTLGN}\label{rdtlgn:sec:formulation}

An R-DTLGN cell is parameterized by four integers: $P$ (input
predicates), $S$ (hidden state trits), $K$ (output trits), and $L$
(layer depth), as described in \cref{fig:r-dtlgn}. At each timestep $t$, the cell receives a predicate
vector $p_t \in \R^P$ and the previous hidden state $h_{t-1} \in \R^S$,
and produces:

\begin{equation*}\label{rdtlgn:eq:cell-update}
\begin{aligned}
    z_t &= [p_t;\; h_{t-1}] \in \R^{P+S}, \;
    \tilde{h}_t = f^{(L)} \circ \cdots \circ f^{(1)}(z_t) \in \R^{S+K}, \\
    h_t &= \clip\bigl(\tilde{h}_t^{(1:S)},\, -1,\, +1\bigr), \;
    y_t = \tilde{h}_t^{(S+1:S+K)},
\end{aligned}
\end{equation*}
where each layer $f^{(\ell)}$ is a collection of two-input PST gates
(Section~\ref{rdtlgn:subsec:pst}), $\clip$ restricts activations to
$[-1,1]$, $h_t$ is the updated hidden state, and $y_t$ is the output.
The hidden state is initialized to $h_0 = (0, \ldots, 0)$, the
all-unknown state under the information ordering.

During training the cell is differentiable (each gate evaluates its
polynomial at continuous values in $[-1,1]$); at inference, the cell
is hardened to an exact ternary logic circuit via trajectory
distillation (Section~\ref{rdtlgn:subsec:traj-distill}).

\subsection{Recurrent Hardening via Trajectory Distillation}
\label{rdtlgn:subsec:traj-distill}

Per-neuron hardening (rounding each polynomial to its nearest gate
independently) described in~\cref{rdtlgn:subsec:pst} is inadequate for recurrent cells: each gate assignment
propagates through all subsequent timesteps, and locally optimal
rounding produces globally incoherent dynamics.

We replace per-neuron rounding with \emph{trajectory distillation}, a
greedy coordinate-descent procedure over a target vocabulary
$\mathcal{V}$. Given $N_{\mathrm{cal}}$ calibration trajectories from
the training data:
(i)~generate teacher verdicts
$\hat{y}^{(n)}_t \in \Tset$ by thresholding the soft network's output;
(ii)~warm-start the circuit by per-neuron rounding against the full
library ($19{,}683$ gates);
(iii)~sweep all gates in output-to-input order, replacing each with
the $\mathcal{V}$-candidate that minimizes verdict disagreement on the
calibration set;
(iv)~iterate until no sweep improves. Each sweep monotonically
decreases disagreement, guaranteeing convergence; $3$--$7$ sweeps
suffice in practice.

\paragraph{Two-phase vocabulary refinement.}
Restricting $\mathcal{V}$ directly to the 17 non-constant
NM$\,\cap\,$IM gates is too constrained, causing significant accuracy
loss. We therefore adopt a two-phase strategy:
\emph{Phase~1} sweeps with $\mathcal{V} = \mathcal{V}_{\mathrm{NM}}$
($172$ NM non-constant gates), eliminating the hardening gap and
guaranteeing stable dynamics
(Section~\ref{rdtlgn:subsec:two-vocabs}).
\emph{Phase~2} attempts to replace each NM-only gate with the nearest
NM$\,\cap\,$IM candidate, accepting the swap if accuracy loss is below
$\eta = 0.1$\,pp. This upgrades as many gates as possible to the
intersection vocabulary, extending degradation guarantees
(Section~\ref{rdtlgn:subsec:degradation}).

\subsection{The Causal STL Monitoring Task}\label{rdtlgn:subsec:learning-task}

Given a bounded STL formula $\varphi$ and a dataset of trajectories
$\{x^{(n)}\}_{n=1}^{N}$ with offline robustness labels
$\rho^{(n)}_t = \rho(\varphi, x^{(n)}, t)$, we seek an R-DTLGN cell
that produces ternary verdicts online from the available observations.
Concretely, at each timestep $t$, the cell observes only the current
predicate values $p_t = (\mu_1(x(t)), \ldots, \mu_P(x(t)))$ and its
own hidden state $h_{t-1}$, and outputs a verdict
$y_t \in \{-1, 0, +1\}$ approximating $\sign(\rho(\varphi, x, t))$.

The \emph{causal baseline} is the causal restriction of an exact STL
evaluator (e.g., STLCG++~\cite{kapoor2025stlcg++}):
it computes the standard robustness using only the available observations
$\{x(0), \ldots, x(t)\}$, producing correct verdicts only for formulas
whose temporal horizon has fully elapsed. The gap between this baseline
and the oracle (full-trajectory) evaluation is the space a learned
monitor must exploit. The R-DTLGN's training objective is a loss on
ternary labels derived from the oracle robustness $\rho^{(n)}_t$, with
two options for label construction described below.

The R-DTLGN is trained on ternary labels derived from the continuous
robustness signal, but two legitimate pipelines exist for constructing
these labels. The distinction is analogous to the
\emph{optimize-then-discretize} (OtD) vs.\
\emph{discretize-then-optimize} (DtO) choice in numerical methods.

\textbf{Compute-then-Quantize (CtQ)} evaluates STL robustness from
the continuous predicates, then quantizes the resulting scalar $\rho$
to a ternary label. \textbf{Quantize-then-Compute (QtC)} first rounds
each predicate to $\Tset$, then evaluates STL robustness from the
ternary predicates. CtQ labels reflect the full continuous-predicate
information; QtC labels reflect only what the ternary inputs can
structurally determine.

Since ternary rounding is sign-preserving
($\sign(\bar{\mu}_i) \in \{0,\,\sign(\mu_i)\}$), the two pipelines
can disagree only by expanding the unknown region:

\begin{prop}[QtC sign preservation]
\label{rdtlgn:prop:conservative-quant}
Let $\bar{\mu}_i : \R^n \to \Tset$ be a sign-preserving quantization
of each predicate. Then for any bounded STL formula $\varphi$ with
standard min/max semantics, the QtC robustness $\bar{\rho}$ satisfies
$\sign(\rho) \cdot \sign(\bar{\rho}) \geq 0$. That is, QtC can only
widen the unknown band relative to CtQ; it cannot flip a satisfaction
to a violation or vice versa.
\end{prop}

\begin{proof}[Proof sketch]
Structural induction on $\varphi$. The base case holds by the
sign-preserving assumption. For $\min$: $\rho > 0$ forces both
arguments positive, so both quantized arguments are $\geq 0$, giving
$\bar{\rho} \geq 0$. $\rho < 0$ forces at least one argument
negative, giving $\bar{\rho} \leq 0$. The $\max$ case is symmetric.
Bounded temporal operators reduce to finite $\min$/$\max$ chains.
\end{proof}

The CtQ/QtC choice is a design tradeoff between accuracy and safety.
CtQ training targets the true verdict but asks the network to commit
at timesteps where ternary inputs alone are insufficient to determine
the outcome; it relies on the learned temporal patterns to fill the
gap. QtC training targets only the structurally recoverable verdicts,
aligning the training objective with the degradation guarantees of
Section~\ref{rdtlgn:subsec:degradation}: the network is never asked to
produce a verdict that its ternary inputs cannot support.
\section{Analysis}\label{rdtlgn:sec:analysis}

We analyze the hardened R-DTLGN cell (the inference-time ternary
circuit). The analysis rests on two gate vocabularies derived from the
two orderings of Section~\ref{rdtlgn:subsec:orderings}: numerically
monotone (NM) gates govern recurrent stability, while
information-monotone (IM) gates govern degradation behavior under input
loss. We restrict the recurrent path to NM gates to promote stable training dynamics. The analysis here focuses on inference-time properties of the hardened circuit.

\subsection{STL Structure and the NM$\,\cap\,$IM Intersection}
\label{rdtlgn:subsec:info-mono}

\begin{prop}[STL recurrent operators use NM$\,\cap\,$IM gates]
\label{rdtlgn:prop:stl-info-mono}
Standard bounded STL temporal operators require \emph{recurrent
connections} composed exclusively of gates in NM$\,\cap\,$IM.
Specifically: $\square_{[a,b]}$ uses windowed AND (min) over a delay
line; $\lozenge_{[a,b]}$ uses windowed OR (max) over a delay line;
and $\varphi_1 \mathbf{U}_{[a,b]} \varphi_2$ decomposes into AND/OR
combinations with delay-line state. Negation ($\neg$) uses the
pointwise NOT gate, which is information-monotone but not numerically
monotone; however, STL semantics ensure negation is applied only to
predicates in the feedforward path, never within temporal recurrence.
\end{prop}

\cref{rdtlgn:prop:stl-info-mono} establishes that the NM$\,\cap\,$IM vocabulary is
expressively sufficient for exact STL monitoring: a faithful
implementation \emph{can} be built from these gates alone. The
recurrent connections use AND (for $\square$:
$h_t = p_t \wedge h_{t-1}$) and OR (for $\lozenge$:
$h_t = p_t \vee h_{t-1}$). Negation is pushed to the predicate level
by duality ($\neg\square\varphi \equiv \lozenge\neg\varphi$), acting
in the feedforward path only. The zero-absorbing behavior of AND and
OR at unknown inputs is the correct ternary monitoring semantic (an
Always monitor receiving unknown input \emph{should} produce unknown).
A learned R-DTLGN need not recover this exact structure; the
proposition ensures that restricting the gate vocabulary to
NM$\,\cap\,$IM does not sacrifice expressiveness for the target task.

\begin{remark}[Expressiveness cost]
Restricting to the 197 IM gates or the 172 NM non-constant gates
excludes the vast majority of the $19{,}683$ possible gates, but none
of the excluded gates correspond to any standard logic operation needed
for STL monitoring. The restriction costs zero expressiveness for the
STL monitoring task.
\end{remark}

\subsection{Degradation Guarantees}\label{rdtlgn:subsec:degradation}

We now state the per-timestep properties that information-monotonicity
provides. These are properties of \emph{single function evaluations},
not of dynamical trajectories across multiple timesteps. They govern
what happens when inputs to a working monitor are lost or degraded.
Crucially, these guarantees are conditional: they hold for any
subcircuit composed of IM gates, and apply to the full cell only if
all gates in the cell are IM.

Let $F : \Tset^P \times \Tset^S \to \Tset^S$ denote the hardened
R-DTLGN cell's state-update function, where $P$ is the number of
predicate inputs and $S$ the number of hidden-state trits.

\begin{prop}[Principled abstention]
\label{rdtlgn:prop:abstention}
If $F$ is composed entirely of non-constant information-monotone gates,
then $F(\bot, \bot) = \bot$. That is, when all predicate inputs are
unknown and the hidden state is all-unknown, the cell produces
all-unknown outputs. The cell never fabricates a verdict from absent
evidence.
\end{prop}

\begin{proof}
For a non-constant information-monotone gate $g$, we have
$g(0,0) = 0$: the output at bottom must be at or below all other
outputs under $\sqsubseteq$, and a non-constant gate cannot map $\bot$
to a determined value (doing so would require the output to be constant,
since $g(\bot) \sqsubseteq g(a,b)$ for all $a, b$). Composing across
all $L$ layers: each layer receives all-zero inputs and produces
all-zero outputs, so the all-unknown state propagates to the output.
\end{proof}

\begin{prop}[Input certainty monotonicity]
\label{rdtlgn:prop:input-certainty}
If $F$ is composed entirely of information-monotone gates, then for any
fixed hidden state $h \in \Tset^S$:
\begin{equation}
    p \sqsubseteq p' \quad\Longrightarrow\quad
    F(p, h) \sqsubseteq F(p', h).
\end{equation}
Replacing any predicate value with unknown ($0$) can only make the
cell's output less determined, never flip a determined verdict.
\end{prop}

\begin{remark}[Degradation, not prediction]
\label{rdtlgn:rmk:degradation-not-prediction}
Propositions~\ref{rdtlgn:prop:abstention}
and~\ref{rdtlgn:prop:input-certainty} are \emph{degradation}
guarantees: architectural invariants of the IM vocabulary that hold for
any learned parameterization. They govern how the monitor behaves when
input information is lost, but say nothing about prediction from
temporal patterns, which comes from the learned parameters. This
separates the R-DTLGN from standard recurrent monitors: in an RNN,
degradation behavior is at best \emph{encouraged by training} (e.g.,
via dropout augmentation) and depends on the learned weights; in the
R-DTLGN it is \emph{guaranteed by architecture} and independent of
them.
\end{remark}

\begin{remark}[Polynomial short-circuit property]
\label{rdtlgn:rmk:short-circuit}
The degradation guarantees hold exactly at ternary grid points. During
training, the PST polynomial surrogates preserve this behavior at
absorbing inputs: if a gate's truth table has a constant row at grid
point $a^*$, the degree-$(2,2)$ interpolant satisfies
$p_g(a^*, b) = c$ for all $b \in \R$ (a degree-2 polynomial with 3
roots is identically zero). This short-circuit property ensures that
the soft network's behavior at ternary inputs matches the hardened
cell's, bridging the training and inference phases for the degradation
analysis.
\end{remark}

\subsection{Initialization via Kleene Fixed-Point Theorem}
\label{rdtlgn:subsec:cross-timestep}

The per-timestep guarantees hold at every evaluation independently. We
now examine what can be said about the cell's behavior across multiple
timesteps when the gates are information-monotone.

\begin{thm}[Kleene initialization convergence]
\label{rdtlgn:thm:kt-applied}
Let $F$ be composed of information-monotone gates and fix a constant
input $p \in \Tset^P$. Define $G_p(h) = F(p, h)$. Then $G_p$ is
monotone on the finite pointed poset $(\Tset^S, \sqsubseteq)$, and the
Kleene ascending chain from the all-unknown state converges to the
least fixed point:
\[
    \bot \sqsubseteq G_p(\bot) \sqsubseteq G_p^2(\bot) \sqsubseteq
    \cdots \sqsubseteq G_p^k(\bot) = h^*
    \quad\text{for some } k \leq S.
\]
The bound $k \leq S$ holds because each strict step promotes at least
one component from $0$ to $\pm 1$, and information-monotonicity prevents
a determined component from reverting to unknown.
\end{thm}

\begin{proof}
The partially ordered set $(\Tset^S, \sqsubseteq)$ is finite and
possesses a unique bottom element $\bot = (0,\ldots,0)$.
$G_p$ is monotone by the same argument as
Proposition~\ref{rdtlgn:prop:input-certainty} (monotonicity in $h$ for
fixed $p$). By Kleene's fixed-point theorem for monotone maps on
finite pointed domains, the ascending chain from $\bot$ converges to the least fixed
point in at most $S$ steps (the height of the lattice).
\end{proof}

\paragraph*{Scope and practical interpretation.}
Theorem~\ref{rdtlgn:thm:kt-applied} assumes constant input $p$ and
convergence from $\bot$. From non-$\bot$ states, convergence is not
guaranteed: IM gates can produce limit cycles (the
gate NOT satisfies $\text{NOT}(0) = 0$, so $\bot$ is a fixed point, but
generates a period-2 cycle
$+1 \to -1 \to +1 \to \cdots$ from determined states, since $+1$ and
$-1$ are incomparable under $\sqsubseteq$). The initialization
$h_0 = \bot$ is therefore a requirement, not merely a convention.

The constant-input assumption is more permissive than it appears. After
ternary quantization, predicates encoding semantically stable conditions
(``the agent is safe,'' ``the agent has reached the goal'') produce
piecewise-constant input sequences whose plateaus typically span many
sampling periods. During each plateau,
Theorem~\ref{rdtlgn:thm:kt-applied} guarantees convergence to the least
fixed point within $S$ steps; when a predicate transitions, the
per-timestep guarantees of
Propositions~\ref{rdtlgn:prop:abstention}--\ref{rdtlgn:prop:input-certainty}
continue to hold during the transient. This characterization weakens for
predicates that chatter near a quantization threshold.

\subsection{Minimum State Dimension}\label{rdtlgn:subsec:state-bound}

We now ask: given a bounded STL formula, how large must the R-DTLGN
cell be? The question has classical precedents: the Myhill-Nerode
theorem relates language complexity to minimum automaton state count,
and the Kalman minimal realization theorem relates input-output
behavior to minimum state dimension in linear systems.
We derive an analogous result for the R-DTLGN.

\begin{definition}[State complexity]
\label{rdtlgn:def:B-phi}
For a bounded STL formula $\varphi$, define the \emph{state complexity}
$B(\varphi)$ recursively. Let $w = b{-}a{+}1$ denote the interval width:
\begin{equation}\label{rdtlgn:eq:B-phi}
\begin{aligned}
    B(\mu) &= 0, \\
    B(\neg\varphi) &= B(\varphi), \\
    B(\varphi_1 \wedge \varphi_2) &= B(\varphi_1) + B(\varphi_2), \\
    B(\square_{[a,b]}\,\varphi) &= w + B(\varphi), \\
    B(\lozenge_{[a,b]}\,\varphi) &= w + B(\varphi), \\
    B(\varphi_1 \,\mathbf{U}_{[a,b]}\, \varphi_2) &=
        2w + B(\varphi_1) + B(\varphi_2).
\end{aligned}
\end{equation}
Disjunction is treated identically to conjunction.
\end{definition}

The intuition is straightforward. Atomic predicates require no memory
(they are external inputs). Negation is pointwise and adds no state.
Conjunction and disjunction require independent buffers for each
subformula. Each temporal operator maintains a shift register of length
$w$ over the subformula's recent evaluations; Until requires two such
registers (one per operand).

\begin{thm}[Minimum state dimension]
\label{rdtlgn:thm:min-state}
Any R-DTLGN cell capable of computing the monitoring function for a
bounded STL formula $\varphi$ requires hidden state dimension
$S \geq B(\varphi)$.
\end{thm}

\begin{proof}[Proof sketch]
The monitoring function must distinguish all signal prefixes that lead
to different future monitoring states. Each temporal operator with
interval $[a,b]$ requires a shift register of length $w$ to store the
subformula's recent evaluations, and the cell's ternary state
components are the only available memory between timesteps. The
registers for distinct operators in the parse tree must be maintained
independently. The additive total across the parse tree gives
$B(\varphi)$.
\end{proof}

\begin{remark}[Necessity, not sufficiency]
The bound $S \geq B(\varphi)$ is necessary: a cell with $S < B(\varphi)$
lacks sufficient memory. It does not claim that $S = B(\varphi)$
suffices. The gap between necessity and sufficiency arises from the
two-input gate topology, which introduces routing overhead.
\end{remark}

\subsection{Computational Depth Bound}\label{rdtlgn:subsec:depth-bound}

The state complexity $B(\varphi)$ bounds the memory (hidden state
dimension) required across timesteps. A complementary bound governs the
per-timestep computation: how many layers of PST gates are needed to
update all sliding windows and produce the current output?

\begin{prop}[Per-timestep depth]
\label{rdtlgn:prop:depth}
For a bounded STL formula $\varphi$ with nesting depth $d$ and maximum
interval length $k_{\max} = \max_{[a,b] \in \varphi}(b{-}a{+}1)$, the
minimum layer count satisfies
\begin{equation}\label{rdtlgn:eq:depth-bound}
    L \;\geq\; d \cdot \lceil \log_2 k_{\max} \rceil.
\end{equation}
\end{prop}

\noindent Together, Theorem~\ref{rdtlgn:thm:min-state} and
Proposition~\ref{rdtlgn:prop:depth} characterize the minimum cell
dimensions: $S \geq B(\varphi)$ gives the width (hidden state), and
$L \geq d \cdot \lceil \log_2 k_{\max} \rceil$ gives the depth
(layers). For the specification
$\varphi_1 = \mu_{\mathrm{safe}} \,\mathbf{U}_{[0,5]}\,
\mu_{\mathrm{goal}}$ ($w = 6$, $d = 1$): $B(\varphi_1) = 12$ and
$L \geq \lceil \log_2 6 \rceil = 3$.

\section{Results}\label{rdtlgn:sec:results}

\newcommand{\tbd}{\textcolor{gray}{--}}
\newcommand{\specfont}[1]{{\small #1}}

\begin{figure}[ht]
    \centering
    \includegraphics[width=\linewidth]{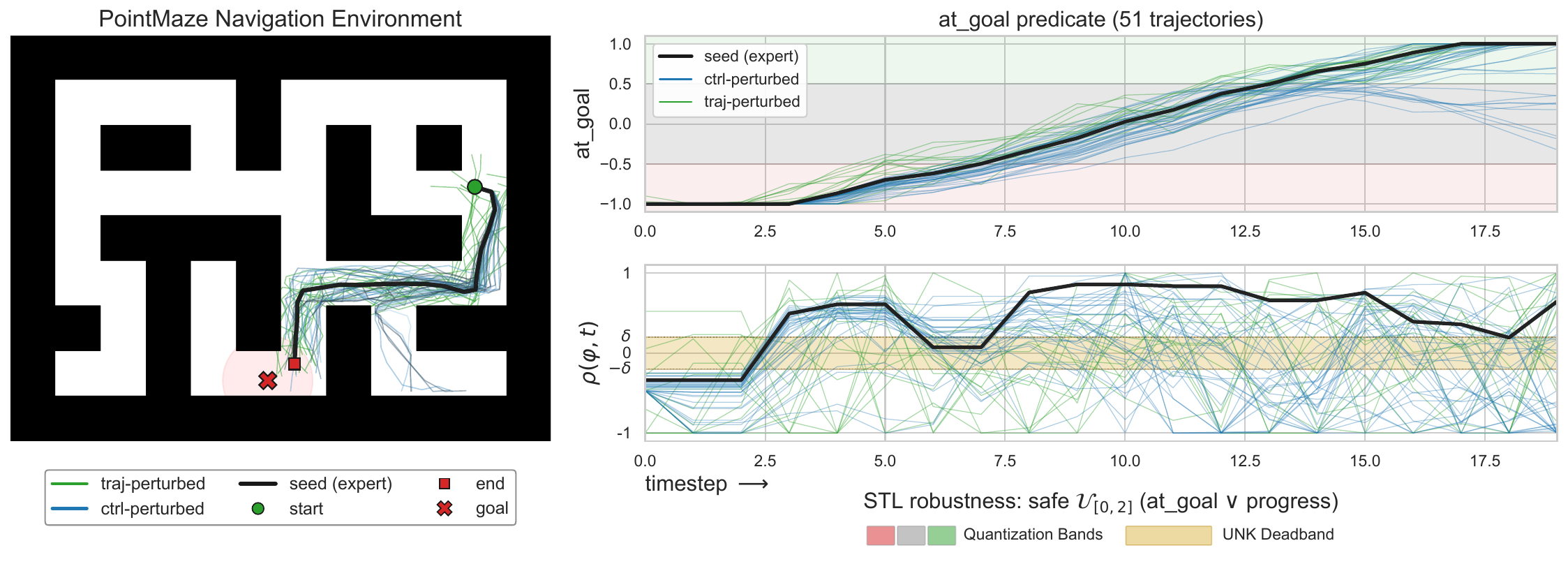}
    \caption{Experimental setup for the navigation case.}
    \label{fig:setup}
    \vspace{-0.25cm}
\end{figure}

\begin{figure}[t]
    \centering
    \includegraphics[width=\linewidth]{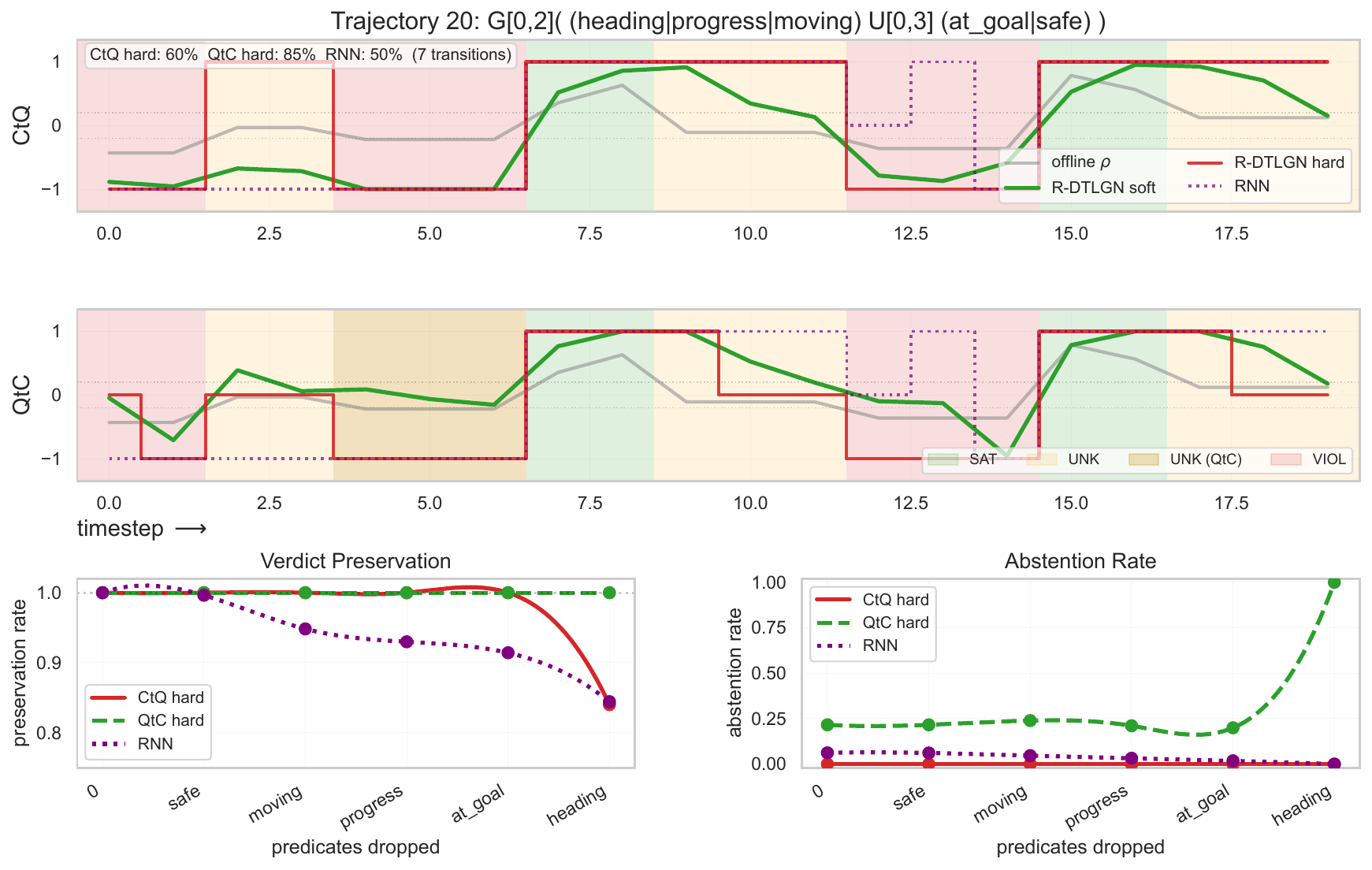}
    \caption{S06 ($P{=}5$), Trajectory 20: verdict traces (top: CtQ, middle: QtC) and
    per-predicate progressive dropout (bottom). The QtC hard circuit
    maintains perfect verdict preservation across all dropout levels
    while abstention increases monotonically
    (Props.~\ref{rdtlgn:prop:abstention}--\ref{rdtlgn:prop:input-certainty});
    the RNN shows degrading preservation with no abstention mechanism.
    Background shading: green = SAT, yellow = UNK, red = VIOL.}
    \label{fig:verdict-trace}
    \vspace{-0.5cm}
\end{figure}

We evaluate the R-DTLGN on $816$ trajectories ($T = 20$) from the D4RL
PointMaze Large maze~\cite{fu2020d4rl}, generated by perturbing seed
trajectories through physics-based resimulation and selected to ensure
balanced ternary label coverage ($\delta = 0.20$). Five predicates are
defined over $x(t) = [q, \dot{q}]$: $\mu_\text{g}$ (at goal),
$\mu_\text{s}$ (safe), $\mu_\text{m}$ (moving), $\mu_\text{h}$
(heading toward goal), $\mu_\text{p}$ (approach rate), each normalized
to $[-1,+1]$.

We construct $6$ bounded STL specifications over these predicates
(Table~\ref{rdtlgn:tab:aggregate}), spanning $2$--$4$ predicates with
closed-prefix Until semantics. All temporal horizons satisfy
$w \leq 4$. Each R-DTLGN cell is sized from the realizability bound:
$S = B(\varphi)$ hidden-state trits, $L = 6$ layers
(Theorem~\ref{rdtlgn:thm:min-state},
Proposition~\ref{rdtlgn:prop:depth}), replacing hyperparameter search
with formula-driven architecture sizing.

Each model is trained with NM-enforced PST
($\lambda_{\max} = 0.3$) under both label construction pipelines
(CtQ and QtC), then hardened via two-phase trajectory distillation
(Section~\ref{rdtlgn:subsec:traj-distill}) using the training set as
calibration data. A vanilla RNN (Elman network) with matched hidden
dimension is trained on CtQ labels as a baseline.

\begin{table*}[t]
\centering
\caption{Evaluation across 6 STL specifications ($T{=}20$, $L{=}6$,
$S{=}B(\varphi)$). \textbf{Prediction:} \emph{Causal} is the
non-learned STLCG++ baseline restricted to available observations;
\emph{RNN} is a vanilla Elman network; \emph{CtQ}/\emph{QtC} are
hard R-DTLGN circuits after two-phase distillation (soft accuracy in
parentheses). \textbf{Preservation:} mean verdict preservation under
single-predicate dropout
(Prop.~\ref{rdtlgn:prop:input-certainty} predicts $1.0$ for IM cells).
\textbf{Lattice:} compliance with input-certainty monotonicity across
all $2^P$ predicate subsets; for each subset pair $A \subset B$,
the less-informed verdict must either abstain or agree with the
more-informed verdict
(Prop.~\ref{rdtlgn:prop:input-certainty} predicts $100\%$
for fully IM cells). Not applicable to the RNN.}
\label{rdtlgn:tab:aggregate}
\vspace{2pt}
\scriptsize
\setlength{\tabcolsep}{2.4pt}
\begin{tabular}{@{}cl l cc c ccc ccc cc@{}}
\toprule
 & & & & & & \multicolumn{3}{c}{\textbf{Prediction (\%)}} &
   \multicolumn{3}{c}{\textbf{Preservation (\%)}} &
   \multicolumn{2}{c@{}}{\textbf{Lattice (\%)}} \\
\cmidrule(lr){7-9} \cmidrule(lr){10-12} \cmidrule(l){13-14}
 & \textbf{Specification} & \textbf{Task} & $P$ & $S$ & \textbf{Causal} &
  \textbf{CtQ} & \textbf{QtC} & \textbf{RNN} &
  \textbf{CtQ} & \textbf{QtC} & \textbf{RNN} &
  \textbf{CtQ} & \textbf{QtC} \\
\midrule

S01 & $\square_{[0,3]}(\mu_\text{h}\;\mathbf{U}_{[0,3]}\;\mu_\text{g})$
    & head $\to$ goal
    & 2 & 12 & 45.4 & 67.7\,(67.7) & 55.4\,(77.9) & 73.4 & 89.3 & 100.0 & 90.0 & 81.0 & 100.0 \\
\midrule

S02 & $\square_{[0,2]}((\mu_\text{h}{\vee}\mu_\text{p})\;\mathbf{U}_{[0,3]}\;\mu_\text{g})$
    & approach $\to$ goal
    & 3 & 11 & 43.7 & 76.7\,(78.7) & 76.6\,(75.9) & 76.5 & 89.2 & 88.7 & 92.8 & 81.0 & 80.3 \\
\midrule

S03 & $\square_{[0,2]}((\mu_\text{h}{\vee}\mu_\text{s})\;\mathbf{U}_{[0,3]}\;(\mu_\text{g}{\vee}\mu_\text{p}))$
    & safe $\to$ arrive
    & 4 & 11 & 29.9 & 68.8\,(69.0) & 64.4\,(71.1) & 66.9 & 93.3 & 100.0 & 91.0 & 92.3 & 100.0 \\
S04 & $\square_{[0,2]}((\mu_\text{h}{\vee}\mu_\text{p})\;\mathbf{U}_{[0,3]}\;(\mu_\text{g}{\vee}\mu_\text{s}))$
    & approach $\to$ safe
    & 4 & 11 & 46.2 & 73.2\,(74.1) & 55.8\,(74.9) & 71.5 & 88.7 & 95.3 & 91.3 & 78.3 & 91.1 \\
\midrule

S05 & $\square_{[0,2]}((\mu_\text{h}{\vee}\mu_\text{s})\;\mathbf{U}_{[0,3]}\;\mu_\text{g}) \wedge \lozenge_{[0,3]}(\mu_\text{p}{\vee}\mu_\text{m})$
    & safe $\to$ goal $+$ move
    & 5 & 15 & 30.1 & 61.0\,(67.8) & 61.3\,(63.6) & 70.3 & 92.9 & 96.4 & 88.5 & 85.9 & 87.0 \\
S06 & $\square_{[0,2]}((\mu_\text{h}{\vee}\mu_\text{p}{\vee}\mu_\text{m})\;\mathbf{U}_{[0,3]}\;(\mu_\text{g}{\vee}\mu_\text{s}))$
    & active $\to$ safe
    & 5 & 11 & 33.8 & 71.7\,(71.0) & 70.4\,(69.4) & 69.7 & 96.8 & 100.0 & 94.2 & 93.9 & 99.7 \\
\midrule
\multicolumn{3}{@{}l}{\textbf{Mean}} & & &
  38.2 & 69.8 & 64.0 & 71.4 & 91.7 & 96.7 & 91.3 & 85.4 & 93.0 \\
\bottomrule
\end{tabular}
\vspace{-8pt}
\end{table*}

\paragraph{Crossing the causality gap}
The nested temporal structure creates a large gap between what a causal
evaluator can resolve and what the monitoring task requires: the causal
STLCG++ baseline averages $38.2\%$ across the 6 specifications, with
values as low as $29.9\%$ (S03). Both the R-DTLGN and the RNN bridge
this gap through learned temporal patterns. On S02, the CtQ hard
circuit reaches $76.7\%$ against a causal floor of $43.7\%$, a
$33$\,pp improvement; on S06, the CtQ hard circuit reaches $71.7\%$
against $33.8\%$, a $38$\,pp improvement. CtQ matches or exceeds the
RNN on four of six specifications (S02, S03, S04, S06), while operating
in a discrete ternary domain that admits structural degradation analysis.

\paragraph{Hardening and distillation}
Two-phase trajectory distillation
(Section~\ref{rdtlgn:subsec:traj-distill}) converts the soft PST
network to a hardened ternary circuit. The hardening gap varies
substantially across specifications: CtQ loses less than $2$\,pp on
S02 and S03, while QtC loses $22.5$\,pp on S01 (from $77.9\%$
soft to $55.4\%$ hard). On S06, both pipelines show a negative
hardening gap (hard exceeds soft), suggesting distillation acts as a
regularizer for that specification. NM$\,\cap\,$IM gate compliance
after the second distillation sweep ranges from $76$--$86\%$, with
higher compliance generally correlating with smaller accuracy loss.
The hardened circuit is a pure ternary logic network with no
floating-point arithmetic, directly realizable as an ASIC.

\paragraph{Degradation behavior}
We evaluate input-certainty monotonicity
(Proposition~\ref{rdtlgn:prop:input-certainty}) on the distilled
circuits through two tests. \emph{Verdict preservation} measures
the fraction of timesteps where masking a single predicate to $0$
does not flip the verdict sign. \emph{Lattice monotonicity} tests
the full proposition across all $2^P$ predicate subsets.
The proposition predicts $100\%$ for both metrics on fully IM cells;
in practice, NM$\,\cap\,$IM compliance after distillation ranges from
$76$--$86\%$, so deviations are expected and observed.

\emph{Preservation} is the more consistent result. QtC achieves
$\geq 88.7\%$ preservation on every specification, reaching $100\%$ on
S01, S03, and S06. CtQ preservation ranges from $88.7\%$ to $96.8\%$,
with a mean of $91.7\%$.
In both cases, the R-DTLGN is competitive with the RNN ($88$--$94\%$),
with QtC exceeding the RNN on five of six specifications.

\emph{Lattice monotonicity} tests
Proposition~\ref{rdtlgn:prop:input-certainty} across the full
power-set lattice of predicate subsets (all $2^P$ combinations), not
just single-predicate dropout. For every pair of subsets
$A \subset B$, the verdict under $A$ must either abstain or agree with
the verdict under $B$; violations indicate that removing information
flipped a determined verdict. QtC achieves $93.0\%$ mean lattice
compliance, reaching $100\%$ on two specifications (S01, S03) and
$99.7\%$ on S06.
CtQ averages $85.4\%$, reflecting the weaker unknown-signal coverage
in CtQ training labels. This metric has no RNN analogue: the
information ordering underlying the lattice test is specific to the
ternary domain and does not apply to continuous-valued networks.

Figure~\ref{fig:verdict-trace} illustrates the full degradation
profile on S06 ($P{=}5$), which achieves the highest compliance and
the clearest separation between pipelines. The QtC circuit maintains
perfect preservation across all dropout levels and transitions to full
abstention at blackout. The CtQ circuit preserves $96.8\%$ of verdicts
but produces $0\%$ abstention, reflecting sparser unknown coverage in
CtQ training labels. The RNN degrades steadily ($94.2\%$ preservation,
no abstention).

The CtQ and QtC pipelines occupy distinct points in the
accuracy-assurance tradeoff. CtQ approaches or matches RNN-level
accuracy but loses the structural degradation advantages that
distinguish the R-DTLGN from a standard recurrent network. QtC
preserves those advantages more reliably, at the cost of a larger
hardening gap on some specifications. In both cases, the preservation
behavior is \emph{guaranteed by architecture} for the IM-compliant
portion of the cell, not merely \emph{encouraged by training} as would
be required for an RNN to exhibit similar behavior through dropout
augmentation.
\section{Conclusion}\label{rdtlgn:sec:conclusion}

We analyzed the R-DTLGN as a causal STL monitor and
identified two structural advantages rooted in its native operation in
Kleene's three-valued domain. First, cells composed of
information-monotone gates guarantee graceful degradation under input
loss: principled abstention and input certainty monotonicity ensure that
verdicts degrade toward ``unknown'' rather than flipping to wrong
answers when predicates become indeterminate. Second, the STL formula's
temporal operator structure determines a lower bound on the cell's
hidden state dimension, replacing hyperparameter search with
formula-driven architecture sizing.

These two capabilities complement the R-DTLGN's role as a learned
monitor. Like any neural architecture trained on robustness labels, the
R-DTLGN exploits spatio-temporal patterns to predict verdicts that a
causal evaluator cannot resolve from the available observations alone. The Kleene
connection adds a layer of structural assurance: the prediction
capability bridges the causal gap, and the degradation guarantee ensures
robustness to input-level uncertainty.

\paragraph*{Limitations and future work}
The degradation guarantees require information-monotone gates;
two-phase distillation improves but does not fully close the hardening
gap. Three directions follow: gated R-DTLGN variants for longer
temporal horizons, multi-trit encodings to break the quantization
bottleneck, and DFA extraction for formal model checking of the
learned circuit.

\bibliographystyle{IEEEtran}
\bibliography{references}

\end{document}